\documentclass[11pt]{article}

\usepackage{amsmath}
\usepackage{amsthm}
\usepackage{amssymb}
\usepackage{algorithm}
\usepackage{color}
\usepackage{graphicx}
\usepackage{wrapfig,epsfig}
\usepackage{url}
\usepackage{color}
\usepackage{algpseudocode}
\usepackage[T1]{fontenc}
\usepackage{bbm}
\usepackage{comment}
\usepackage{dsfont}
\usepackage{thm-restate}
\usepackage{bm}
\usepackage{tcolorbox}
\usepackage{subcaption}
\usepackage{enumitem}

\usepackage[margin=1in]{geometry}

\graphicspath{{./figs/}}

\newtheorem{theorem}{Theorem}[section]
 
\newtheorem{lemma}[theorem]{Lemma}
\newtheorem{definition}[theorem]{Definition}

\newtheorem*{remark*}{Remark}

\newcommand{\wt}{\widetilde}
\newcommand{\eps}{\epsilon}

\newcommand{\R}{\mathbb{R}}
\newcommand{\X}{\mathcal{X}}
\newcommand{\mH}{\mathcal{H}}
\newcommand{\Y}{\mathcal{Y}}
\newcommand{\D}{\mathcal{D}}
\newcommand{\C}{\mathcal{C}}
\newcommand{\mI}{\mathcal{I}}

\DeclareMathOperator{\OPT}{OPT}
\DeclareMathOperator{\ALG}{ALG}

\DeclareMathOperator*{\argmin}{arg\,min}

\title{The sample complexity of multi-distribution learning}
\author{}
\author{  
Binghui Peng \\ Columbia University \\ \texttt{bp2601@columbia.edu}
}
\date{}

\begin{document}
\maketitle

\begin{abstract}
Multi-distribution learning generalizes the classic PAC learning to handle data coming from multiple distributions. Given a set of $k$ data distributions and a hypothesis class of VC dimension $d$, the goal is to learn a hypothesis that minimizes the maximum population loss over $k$ distributions, up to $\eps$ additive error.
In this paper, we settle the sample complexity of multi-distribution learning by giving an algorithm of sample complexity $\wt{O}((d+k)\eps^{-2}) \cdot (k/\eps)^{o(1)}$. 
This matches the lower bound up to sub-polynomial factor and resolves the COLT 2023 open problem of Awasthi, Haghtalab and Zhao \cite{awasthi2023open}.
\end{abstract}

\setcounter{page}{0}
\thispagestyle{empty}

\newpage

\section{Introduction}
\label{sec:intro}

Multi-distribution learning is a natural generalization of the classic PAC learning \cite{valiant1984theory} to multiple distributions setting. Given a hypothesis class $\mH$ and a set of $k$ distributions $\D_{1}, \ldots, \D_{k}$ over the data universe $\X \times \{0,1\}$, multi-distribution learning seeks for a hypothesis $f$ that achieves near optimal worst case guarantee over all distributions
\begin{align*}
\max_{i \in [k]}\ell_{\D_i}(f) \leq \argmin_{h^{*} \in \mH}\max_{i \in [k]}\ell_{\D_i}(h^{*}) + \eps \quad \text{where}\quad \ell_{\D}(h):= \Pr_{(x, y)\sim \D}[h(x) \neq y].
\end{align*}
The formulation of multi-distribution learning captures many important applications: For fairness consideration, the distributions represent heterogeneous populations of protected attributes and multi-distribution learning yields the minimax group fairness \cite{mohri2019agnostic, shekhar2021adaptive, rothblum2021multi,diana2021minimax,tosh2022simple}; 
In the context of multi-task or federated learning, multi-distribution learning captures the notion of robustness and yields worst case guarantees \cite{sener2018multi}; 
For group distributional robustness optimization, multi-distribution learning obtains uniform guarantee to all pre-defined groups of distributions \cite{rahimian2019distributionally, sagawa2019distributionally, sagawa2020investigation, duchi2021learning}.

Similar to the study of PAC learning \cite{blumer1989learnability, auer2004new, hanneke2016optimal,larsen2023bagging, aden2023optimal}, one important research question is to characterize the sample complexity of multi-distribution learning. It is not hard to see that the learnability is still captured by the VC dimension \cite{vapnik1971uniform}, and $\wt{\Theta}(kd/\eps^2)$ samples are both necessary and sufficient to guarantee uniform convergence \cite{blum2017collaborative}. There is a long line of work \cite{blum2017collaborative, nguyen2018improved, chen2018tight, haghtalab2022demand, awasthi2023open} that try to pin down the optimal sample complexity. In the realizable setting, where the optimal hypothesis $h^{*}\in \mH$ has zero error, \cite{blum2017collaborative, nguyen2018improved, chen2018tight} give algorithms of sample complexity $\wt{O}((d+k)/\eps)$ using the idea of multiplicative weight update. 
The sample complexity for the general agnostic learning setting is more challenging. A recent breakthrough of Haghtalab, Jordan and Zhao \cite{haghtalab2022demand} gives an algorithm of sample complexity $\wt{O}((k+\log(|\mH|))/\eps^2)$, this is optimal assuming the hypothesis class $\mH$ is finite.
For infinite hypothesis class, \cite{awasthi2023open} gives two algorithms: One bases on the multiplicative weight update and has sample complexity $\wt{O}((d+k)/\eps^4)$; The other bases on the finite hypothesis algorithm \cite{haghtalab2022demand} and has sample complexity $\wt{O}((d+k)/\eps^2 + kd/\eps)$.
Nevertheless, there is still a significant gap between the upper and lower bound ($\wt{O}(\min\{(d+k)/\eps^4, (d+k)/\eps^2 + kd/\eps\})$ vs. $\wt{\Omega}((d+k)/\eps^2)$), and as elaborated in the COLT 2023 open problem publication \cite{awasthi2023open}, fundamental barriers exist for all current approaches. They pose the open question of obtaining the optimal sample complexity for multi-distribution learning. 

In this paper, we address the open question of \cite{awasthi2023open} and give an algorithm of optimal sample complexity (up to sub-polynomial factor). Our result is formally stated as below.
\begin{theorem}[Multi-distribution learning]
\label{thm:multi}
Let $k$ be the number of distributions, $d$ be the VC dimension of the hypothesis class. For any $\eps > 0$, there is an algorithm that outputs an $\eps$-optimal classifier with probability $1-\delta$, and has sample complexity 
\[
\frac{(d+k)\log(d/\delta) }{\eps^{2}} \cdot  (k/\eps)^{o(1)}.
\]
\end{theorem}
An immediate implication of Theorem \ref{thm:multi} is that {\em multi-distribution learning is no harder than (single-distribution) PAC learning} for sample complexity consideration.

\subsection{Technical overview: Achieving optimal sample complexity via recursive width reduction}

We give an overview of our algorithm for Theorem \ref{thm:multi}, the key ingredient is a recursive width reduction procedure.

\paragraph{The MWU framework} The major technique used by all previous works \cite{blum2017collaborative, nguyen2018improved, chen2018tight,haghtalab2022demand, awasthi2023open} is the multiplicative weight update (MWU) framework \cite{arora2012multiplicative}. We first review this framework. 
The algorithm views the $k$ distributions $\D_1, \ldots, \D_k$ as $k$ experts and runs MWU for $T$ rounds. At each round $t \in [T]$, the algorithm performs empirical risk minimization (ERM) and obtains an $\eps$-optimal hypothesis $f_t \in \mH$ over the mixed distribution $\D^{(t)} = \sum_{i\in [k]}p_t(i) \D_i$. 
Here $p_t$ is the strategy of MWU, and it is updated by the loss of $f_t$ over distributions $(\D_{i})_{i\in [k]}$, i.e. $\ell_t = (\ell_{\D_i}(f_t))_{i\in [k]} \in [0,1]^k$. The final output is taken to be $f = \frac{1}{T}\sum_{t\in [T]}f_t$. 
The regret guarantee of MWU ensures that the worst case error of $f$ over $\D_1, \ldots, \D_k$ is close to the average error of $f_t$ over $\D^{(t)}$, which is at most $\eps$. For the sample complexity, in the realizable setting, the sample complexity per round is $\wt{O}((k+d)/\eps)$ and $T= \Omega(\log(k))$ rounds are needed; under the agnostic learning setting, the sample complexity per round is $\wt{O}((k+d)/\eps^2)$ and $T = \wt{\Omega}(1/\eps^2)$ rounds are needed. It is not hard to see that both terms are tight in the worst case and they form the major technical obstacle for obtaining the optimal sample complexity.

\paragraph{Width reduction} Our approach also falls into this MWU framework and the key idea for improvement is {\em recursive width reduction}. In the literature of online learning, width refers to the maximum range of the loss vector, i.e., $\max_{i\in [k]}\ell_t(i) - \min_{i\in [k]}\ell_t(i)$. 
To get a quick sense of how width reduction works, recall the regret guarantee of MWU equals $O(\sqrt{\log(k)/T}B)$, where $B$ is the width of the loss vector $\ell_t$ and it equals $1$ in the above framework. In order to get $\eps$-regret, one needs to take $T \geq \wt{\Omega}(1/\eps^2)$. If one can reduce the width, then it immediately reduces the number of iterations for MWU, and consequently, reduces the sample complexity.

For now, we assume the optimal error $\OPT:= \min_{h^* \in \mH}\max_{i \in [k]}\ell_{\D_i}(h^{*})$ is known to the algorithm, this assumption can be easily removed and we defer the discussion to the end.  Our idea is to reduce the width using the algorithm itself. Recall we have an algorithm of sample complexity $\wt{O}((d+k)/\eps^{4})$ using MWU. Now at each round $t\in [T]$, we first draw $\wt{O}(d/\eps^2)$ samples from $\D^{(t)}$. Instead of running ERM, we first obtain a subset of the hypothesis $\mH' \subseteq \mH$ by removing all hypothesis $h \in \mH$ that has error more than $\OPT + \eps$.
We then run the MWU algorithm with error parameter $\eps' = \eps^{1/2}$ over hypothesis class $\mH'$ (so the additional samples it needs is still $\wt{O}((d+k)/(\eps')^4) = \wt{O}((d+k)/\eps^2)$), and obtain a hypothesis $f_t$. The hypothesis $f_t$ has additional guarantees on the maximum loss, i.e., $\max_{i \in [k]}\ell_{\D_i}(f_t) \leq \OPT + \eps^{1/2}$. This reduces the maximum loss from $1$ to $\OPT+ \eps^{1/2}$. However, we have no guarantee on the minimum loss, and the width could still be as large as $\OPT+\eps^{1/2} \approx \Theta(1)$.

To get a lower bound on $\ell_t$, we can truncate small entries: If an entry $\ell_t(i) = \ell_{\D_i}(f_t) \leq \OPT - \eps^{1/2}$ is small, then we take it as $\ell_t(i) = \OPT -\eps^{1/2}$. In this way, the width reduces to $(\OPT+\eps^{1/2}) - (\OPT - \eps^{1/2}) = 2\eps^{1/2}$. However, there is a fatal issue here: There is no reason we can arbitrarily truncate the loss. Recall we need the average loss $\sum_{i \in [k]}p_t(i)\ell_{t}(i)$ to be close to $\OPT+\eps$. If we truncate the small entries, then we increase its value. As a concrete example, if $\OPT = 1/2$, there are $2\eps^{1/2}$-fraction of $(\ell_{\D_i}(f_t))_{i\in [k]}$ equal $0$, and the other $1-2\eps^{1/2}$ fraction equal $\OPT+\eps^{1/2}$, then there is no way one can truncate the loss.

The next idea is, instead of relying on uniform convergence and selecting the hypothesis class $\mH'$ that are $\eps$-optimal on $\D^{(t)}$, we need more refined properties of $\mH'$ that make the loss  $(\ell_{\D_i}(f_t))_{i\in [k]}$ more balanced. 
To this end, we want the hypothesis class $\mH'$ satisfies the following two properties:
\begin{itemize}
\item \textbf{Soundness.} The optimal classifier survives, i.e., $h^{*} \in \mH'$, where $h^{*} = \argmin_{h\in \mH}\max_{i\in [k]}\ell_{\D_i}(h)$.
\item \textbf{Completeness.} For any hypothesis $h \in \mH'$ that survives, it satisfies the following guarantee. For any subset of distributions $\mI \subseteq [k]$, if their weights $\sum_{i\in [n]}p_t(i) \geq 1/2$, then the loss of $h$ on the distribution $\sum_{i\in \mI} \frac{p_t(i)}{\sum_{i\in\mI}p_t(i)}\D_{i}$ is at most $\OPT + O(\eps)$.
\end{itemize}
The first property states that the optimal classifier $h^{*}$ survives, this ensures that it is safe to work with $\mH'$ instead of $\mH$. The second property is more complicated, but from a high level, it says that any surviving hypothesis in $\mH'$ is robust -- not only their loss is small on the entire distribution $\D^{(t)} = \sum_{i\in [k]}p_t(i)\D_i$, but also small on any sub-populations $\{\D_{i}\}_{i \in \mI}$ of mass at least $1/2$. Suppose for now, we have achieved these two properties with $\wt{O}((d+k)/\eps^2)$ samples. Then we can safely truncate the loss $\ell_{t}(i) = \max\{\ell_{\D_i}(f_t), \OPT - \eps^{1/2}\}$ and reduce the width of $\ell_{t} \in [\OPT-\eps^{1/2}, \OPT+\eps^{1/2}]^{k}$ to $2\eps^{1/2}$. It remains to argue that the average loss satisfies $\sum_{i\in [k]}p_{t}(i)\ell_{t}(i) \leq \OPT + O(\eps)$.\footnote{We also need to ensure $\ell_t(i)$ is an overestimate of $\ell_{\D_i}(f_t)$, but this is trivial from the definition.} 
To this end, we sort the loss $\{\ell_{\D_i}(f_t)\}_{i \in [k]}$ and assume $\ell_{\D_1}(f_t)\geq \cdots \geq \ell_{\D_k}(f_t)$ w.l.o.g. Suppose $k' \in [k]$ is the smallest index such that $\sum_{i\leq k'}p_{t}(i) \geq 1/2$.
\begin{itemize}
\item \textbf{Case 1.} If $\ell_{\D_{k'}}(f_t) \geq \OPT - \eps^{1/2}$ (i.e., no truncation at the larger half), then by the completeness property,\footnote{We actually need the hypothesis $f_{t}$ to be a weighted average of hypothesis in $\mH'$ in order to inherit the completeness property of $\mH'$, this is naturally satisfied by algorithms in the MWU framework.} the larger half has loss at most $\OPT+O(\eps)$, since there is no truncation. Meanwhile, the loss of the smaller half is no more than the larger half, so the average loss is at most $\OPT + O(\eps)$.
\item \textbf{Case 2.} Otherwise, if $\ell_{\D_{k'}}(f_t) < \OPT - \eps^{1/2}$, then performing truncation is still fine, because more than $1/2$-fraction of distributions have loss smaller than $\OPT - \eps^{1/2}$, while the rest of them (at most $1/2$-fraction) have loss at most $\OPT+\eps^{1/2}$. 
\end{itemize}

Now we elaborate a bit on how we achieve both soundness and completeness. It proceeds in two steps. First, we draw $\wt{O}(d/\eps)$ samples $S^{(t)}_1$ from $\D^{(t)}$ and look at the projection of $\mH$ on $S^{(t)}_1$. We construct an $\eps$-cover $\C_{\mH}$ of $\mH$ by including an arbitrary hypothesis for each projection. This is a fairly standard trick (e.g. see \cite{alon2019limits}).
Next, we draw another $\wt{O}((d+k)/\eps^2)$ samples $S^{(t)}_2$ from $\D^{(t)}$ and run the following test on $S^{(t)}_2$.
For each hypothesis $h \in \C_{\mH}$ in the $\eps$-cover, if there exists a subset $\mI\subseteq [k]$ of distributions such that (1) $\sum_{i \in \mI}p_t(i) \geq 1/2$ and (2) the empirical loss of $h$ on $\sum_{i\in \mI} \frac{p_t(i)}{\sum_{i\in\mI}p_t(i)}\D_{i}$ is larger that $\OPT + O(\eps)$, then we remove $h$, as well as all hypothesis that have the same projection as $h$ on $S_{1}^{(t)}$, from $\mH'$. In the proof, we show that this test guarantees soundness and completeness with high probability.

\paragraph{Recursive width reduction} The width reduction procedure described above reduces the width from $1$ to $2\eps^{1/2}$. The regret now becomes $\wt{O}(\eps^{1/2}/\sqrt{T})$, and it suffices to take $T = \wt{O}(1/\eps)$. The sample complexity per round remains $\wt{O}((d+k)/\eps^2)$ and there are $T = \wt{O}(1/\eps)$ rounds, so we improve the sample complexity from $\wt{O}((d+k)/\eps^4)$ to $\wt{O}((d+k)/\eps^3)$. We can continue this process, and use this new algorithm for width reduction. In particular, at each round, we can take the error parameter $\eps' = \eps^{2/3}$ and use $\wt{O}((d+k)/(\eps')^{3}) = \wt{O}((d+k)/(\eps)^{2})$ samples to reduce the maximum loss to $\OPT +\eps' = \OPT + \eps^{2/3}$ (instead of $\OPT + \eps^{1/2}$). The regret now becomes $\wt{O}(\eps^{2/3}/\sqrt{T})$ and we can further reduce the number of rounds to $T = \wt{O}(\eps^{-2/3})$ and the sample complexity to $\wt{O}((d+k)/\eps^{8/3})$. 
We repeat the above process and obtain an algorithm of sample complexity $O((d+k)/\eps^2) \cdot (k/\eps)^{o(1)}$.

\paragraph{Remove prior knowledge on $\OPT$} The above algorithm requires prior knowledge on the optimal value (for both the testing step and the truncation step), we next remove this assumption. 
It is not hard to see that the above algorithm succeeds with an $\eps$-approximate $\OPT' \in [\OPT -\eps, \OPT +\eps]$ (i.e., no need for the exact value of $\OPT$).
Hence, we can run $1/\eps$ threads of the algorithm with $\OPT'=\eps,2\eps, \ldots, 1$ and take the best one. This has sample complexity $(k+d)/\eps^{2} \cdot (1/\eps) = (k+d)/\eps^3$ (we omit the $o(1)$ term for simplicity) and does not require any knowledge of $\OPT$. 
Next, we take $\eps' = \eps^{2/3}$ and runs the algorithm with $(k+d)/(\eps')^3 = (k+d)/\eps^2$ samples . The output hypothesis has error at most $\OPT + O(\eps') = \OPT + O(\eps^{2/3})$. Now, we can reduce the size of the grid search and only search for $\eps^{2/3}/\eps = \eps^{-1/3}$ possible value of $\OPT'$ (instead of $1/\eps$), this reduces the sample complexity from $(k+d)/\eps^{3}$ to $(k+d)/\eps^{7/3}$. Again, we repeat this process and get an algorithm of sample complexity $O((d+k)/\eps^2) \cdot (k/\eps)^{o(1)}$ without any knowledge of $\OPT$.



\subsection{Related work}
The sample complexity of multi-distribution learning has been extensively studied in the past decade \cite{blum2017collaborative, nguyen2018improved, qiao2018outliers, chen2018tight,tao2019collaborative, blum2021one, haghtalab2022demand, awasthi2023open}. The optimal sample complexity has been derived in the realizable setting \cite{blum2017collaborative, nguyen2018improved, chen2018tight}. For the more general agnostic learning setting, the optimal sample complexity has been obtained for finite hypothesis class \cite{haghtalab2022demand} but the question is widely open for VC classes, we refer interesting readers for the open problem publication of \cite{awasthi2023open} for an excellent coverage on the literature.

The multi-distribution learning has applications to fairness \cite{hebert2018multicalibration, mohri2019agnostic, shekhar2021adaptive, rothblum2021multi, tosh2022simple} and group distributional robust optimization \cite{ben2009robust, rahimian2019distributionally, sagawa2019distributionally, sagawa2020investigation, duchi2021learning}.
It is also closely related to multi-task learning \cite{caruana1997multitask}, distributed learning \cite{balcan2012distributed}, federated learning \cite{mcmahan2017communication}, meta learning \cite{finn2017model} and continual learning \cite{chen2022memory}.

Our approach can be seen as a boosting framework for (agnostic) multi-distribution learning. 
It converts a weak multi-distribution learner into one with better sample complexity guarantee. 
There is a vast literature on boosting \cite{schapire1990strength, freund1997decision, freund1999short,ben2001agnostic,mansour2002boosting,kalai2003boosting, kalai2008agnostic, balcan2012distributed,schapire2013explaining,beygelzimer2015optimal,brukhim2020online,alon2021boosting,brukhim2021multiclass, brukhim2023improper}, but to the best of our knowledge, it is the first time that width reduction has been used -- we hope it could find broad applications for boosting.
The idea of width reduction traces back to the seminal work of positive LP solver \cite{garg2007faster} and approximate max flow \cite{christiano2011electrical}, which use separate subroutines for width reduction. The idea of recursive width reduction (recursively applying the algorithm itself to reduce the width) has been introduced recently by \cite{peng2023online} and it is crucial for the recent development of low memory online learning algorithm \cite{peng2023online, peng2023near}.
These previous work are very inspiring, but our way of width reduction, which forms the major challenging part of the proof, is unique and different.

\paragraph{Concurrent and independent work} We were recently made aware of the concurrent and independent work of Zhang, Zhan, Chen, Du and Lee \cite{zhang2023optimal}, which obtains the similar result as Theorem \ref{thm:multi}. Moreover, their result has the optimal sample complexity up to {\em polylogarithmic} factor, and their algorithm is oracle efficient. Their result is derived via a different set of technique, which relies on sampling reusing.

\section{Preliminary}
\label{sec:pre}

Let $\X$ be the data universe and $\Y = \{0,1\}$ be binary labels. Let $\mH$ be a hypothesis class of VC dimension $d$, a hypothesis $h\in \mH$ maps the data universe $\X$ to the binary label $\Y$.
For any hypothesis $f: \X \rightarrow \Y$ and any distribution $\D$ over $\X\times \Y$, the population loss of $f$ over $\D$ equals
$
\ell_{\D}(f) := \Pr_{(x, y) \sim \D}[f(x) \neq y].
$

\begin{definition}[Multi-distribution learning]
Let $\eps > 0$ be the error parameter.
In the task of multi-distribution learning, there are $k$ distributions $\D_1, \ldots, \D_{k}$ over $\X \times \Y$.
The goal is to learn a hypothesis $f$ that minimizes the maximum loss, i.e.,
\begin{align}
\max_{i \in [k]}\ell_{\D_i}(f) \leq \min_{h^{*} \in \mH}\max_{i \in [k]}\ell_{\D_i}(h^{*})  + \eps. \label{eq:multi-goal}
\end{align}
\end{definition} 
We say an algorithm is an $(\eps, \delta)$-multi-distribution learner if its output satisfies Eq.~\eqref{eq:multi-goal} with probability at least $1 - \delta$. 
In the rest of this paper, we write $h^{*} \in \mH$ to be the hypothesis that obtains the minimum loss and $\OPT$ be the minimum loss, i.e.,
\begin{align*}
h^{*} = \argmin\limits_{h \in \mH}\max_{i \in [k]}\ell_{\D_i}(h) \quad \quad \text{and} \quad \quad \OPT = \max_{i \in [k]}\ell_{\D_i}(h^{*}).
\end{align*}


For any set $S = \{x_1, \ldots, x_n\} \in \X^{d}$, let $\mH(S) := \{(h(x_1), \ldots, h(x_n)) : h \in \mH\}\subseteq \{0,1\}^{n}$ be the projection of $\mH$ onto $S$. The Sauer–Shelah Lemma gives an upper bound on the size $|\mH(S)|$.
\begin{lemma}[Sauer–Shelah Lemma \cite{sauer1972density,shelah1972combinatorial}]
\label{lem:ss}
Let $\mH$ be a hypothesis class with VC dimension $d$, then for any $S\subseteq \X$ with $|S| = n$, $|\mH(S)| \leq \sum_{i=0}^{d}\binom{n}{i}$. In particular,  $|\mH(S)| \leq (en/d)^{d}$ if $n\geq d$.
\end{lemma}

The multiplicative weight updating \cite{littlestone1994weighted} is a classic algorithm for online learning. 
An online learning task can be seen as a repeated game between an algorithm and the nature for a sequence of $T$ rounds. Let $[n] = \{1,2,\ldots, n\}$ and let $\Delta_{n}$ be all probability distributions over $[n]$. 
The MWU algorithm commits a distribution $p_t \in \Delta_n$ over a set of $n$ experts at each round $t \in [T]$, and then the nature reveals the loss $\ell_t \in \R^{n}$ for experts $[n]$. 
The goal is to minimize the regret $\sum_{t\in [T]}\langle p_t, \ell_t\rangle - \min_{i^{*}\in [n]}\sum_{t\in [T]}\ell_t(i^{*})$

\begin{algorithm}[!htbp]
\caption{Multiplicative weight update}
\label{algo:mwu}
\begin{algorithmic}[1] 
\For{$t=1,2, \ldots, T$}
\State Compute $p_t \in \Delta_n$ over experts such that $p_t(i) \propto \exp(-\eta \sum_{\tau=1}^{t-1}\ell_\tau(i))$ for $i\in [n]$
\State Observe the loss vector $\ell_t$ and receives loss $\langle p_t, \ell_t\rangle$
\EndFor
\end{algorithmic}
\end{algorithm}
\begin{lemma}[Regret guarantee of MWU \cite{arora2012multiplicative}]
\label{lem:mwu}
Let $n$ be the number of experts, $T$ be the number of days, $B$ be the width of the loss sequence, i.e., the loss vector $\ell_t \in [\rho_{t},\rho_{t} + B]^n$ at each day $t\in [T]$. Let $\eta = \sqrt{\log(n)/T}/B$ be the learning rate, then the MWU algorithm guarantees
\begin{align*}
    \sum_{t \in [T]} \langle p_t, \ell_t\rangle - \min_{i^{*}\in [n]}\sum_{t\in [T]}\ell_t(i^{*}) \leq \frac{\log n}{\eta} + \eta T B^2 = 2\sqrt{\log(n)T} B.
\end{align*}
\end{lemma}

\section{The boosting framework}
\label{sec:algo}

We provide a general boosting framework that takes an arbitrary multi-distribution learning algorithm, reduces its error while incurring a mild overhead on the sample complexity.
The boosting framework is formally described in Algorithm \ref{algo:boost}. It contains several subroutines, whose pseudocodes are presented in Algorithm \ref{algo:cover}-\ref{algo:estimate}.
The input of $\textsc{BoostLearner}$ (Algorithm \ref{algo:boost}) consists of the hypothesis class $\mH$, $k$ distributions $\D_1, \ldots, \D_k$, a multi-distribution learning algorithm $\textsc{MultiLearnerOracle}$, as well as an estimate $\OPT'$  on the optimal loss.

$\textsc{BoostLearner}$ views $\D_1, \ldots, \D_k$ as $k$ experts, and runs MWU over them for $T$ rounds.
At each round $t \in [T]$, $\textsc{BoostLearner}$ maintains a strategy $p_t \in \Delta_k$ over $k$ data distributions, and let $\D^{(t)} = \sum_{i\in [k]}p_t(i)\D_{i}$ be the mixed distribution of $\D_1, \ldots, \D_k$. 
$\textsc{BoostLearner}$ proceeds in a few steps.


\vspace{+2mm}
{\bf \noindent Construct $\eps$-cover of $\mH$ \ \ } The first step is to construct an $\eps$-cover of the hypothesis $\mH$ on the distribution $\D^{(t)}$ (Line \ref{line:cover} of Algorithm \ref{algo:boost}). 
$\textsc{ConstructCover}$ (Algorithm \ref{algo:cover}) samples $m_1 = \wt{O}(d/\eps)$ data points $S_{1}^{(t)}$ from $\D^{(t)}$, the cover $\C_{\mH} \subseteq \mH$ is constructed by including an arbitrary hypothesis $h \in \mH$ for each projection of $\mH(S_1^{(t)})$. 

\vspace{+2mm}
{\bf \noindent Filter $\mH$ \ \ } Given the $\eps$-cover $\C_{\mH}$, the next step is to filter $\mH$ and only keep a subset of good hypothesis $\mH'\subseteq \mH$ (Line \ref{line:filter} of Algorithm \ref{algo:boost}). The $\textsc{Filter}$ procedure (Algorithm \ref{algo:filter}) draws $m_2 = \wt{O}(\frac{d+k}{\eps^2})$ samples $S^{(t)}_2$ from $\D^{(t)}$ as a test set.
For each hypothesis $h$ in the cover $\C_{\mH}$, it goes through all subsets $\mI$ of $[k]$ . If the probability mass $\sum_{i\in \mI}p_t(i)$ is large enough (i.e., greater than $1/2$) and the empirical loss of $h$ on the mixture distribution $\sum_{i\in\mI}\tfrac{p_{t}(i)}{\sum_{i\in\mI}p_{t}(i)}\D_i$ is large (i.e., great than $\OPT'+8\eps$), then it removes $h$, as well as any hypothesis $h' \in \mH$ that has the same projection as $h$ on $S_{1}^{(t)}$, from $\mH'$.

\vspace{+2mm}
{\bf \noindent Evoke the oracle \ \ } After obtaining the new hypothesis class $\mH' \subseteq \mH$, $\textsc{BoostLearner}$ evokes the oracle $\textsc{MultiLearnerOracle}$ with hypothesis $\mH'$ and distributions $\D_1, \ldots, \D_k$, and obtains a hypothesis $f_t$.

\vspace{+2mm}
{\bf \noindent Construct the loss vector \ \ }
Given the hypothesis $f_t$, $\textsc{BoostLearner}$ constructs the loss vector $\ell_t$ and feeds it to MWU (Line \ref{line:estimate}-\ref{line:update} of Algorithm \ref{algo:boost}).
$\textsc{Estimate}$ (Algorithm \ref{algo:estimate}) draws $\wt{O}(1/\eps^2)$ samples from each distribution $\D_i$ and compute the empirical loss $\hat{\ell}_{\D_i}(f_t)$ of $f_t$ on $\D_i$.
Instead directly using this empirical loss, $\textsc{Estimate}$ further truncates loss entries that are  below $\OPT' - \alpha$ (Line \ref{line:truncate} of Algorithm \ref{algo:estimate}), here $\alpha$ is the error of $\textsc{MultiLearnerOracle}$.

\vspace{+2mm}
{\bf \noindent Final output \ \ } The final output is taken to be the average of $\{f_t\}_{t\in [T]}$. In particular, the output $f = \frac{1}{T}\sum_{t\in [T]}f_t$ is defined as
\[
\Pr[f(x) = 1] = \frac{1}{T}\sum_{t\in [T]}\Pr[f_t(x) = 1] \quad \quad \forall x \in \X.
\]

\begin{algorithm}[!htbp]
\caption{$\textsc{BoostLearner}(\mH, \D_1, \D_2, \ldots, \D_k, \textsc{MultiLearnerOracle}, \OPT')$}
\label{algo:boost}
\begin{algorithmic}[1]
\For{$t = 1,2,\ldots, T$}
\State $\D^{(t)} \leftarrow \sum_{i\in [k]}p_{t}(i)\D_i$  \Comment{$p_t \in \Delta_k$ is the strategy of MWU}
\State $\C_{\mH} \leftarrow \textsc{ConstructCover}(\mH, \D^{(t)})$ \label{line:cover}
\State $\mH' \leftarrow \textsc{Filter}(\mH, \D^{(t)}, \C_{\mH}, \OPT')$ \label{line:filter}
\State $f_t \leftarrow \textsc{MultiLearnerOracle}(\mH', \D_1, \ldots, \D_k, \OPT')$
\State $\ell_t \leftarrow \textsc{Estimate}(f_t)$ \label{line:estimate}
\State Update the strategy of MWU with loss vector $-\ell_t$ \label{line:update}
\EndFor
\State \Return $f = \frac{1}{T}\sum_{t\in [T]}f_t$
\end{algorithmic}
\end{algorithm}

\begin{algorithm}[!htbp]
\caption{$\textsc{ConstructCover}(\mH, \D^{(t)})$}
\label{algo:cover}
\begin{algorithmic}[1]
\State Sample $m_1 = O(\frac{d\log(kd/\eps\delta)}{\eps})$ data points $S_1^{(t)} = \{(x_j, y_j)\}_{j \in [m_1]}$ from $\D^{(t)}$
\State $\C_{\mH} \leftarrow \emptyset$
\For{$(z_1, \ldots, z_{m_1}) \in \mH(S^{(t)}_1)$}  \Comment{$\mH(S^{(t)}_1)$ is the projection of $\mH$ onto $S^{(t)}_1$}
\State Let $h \in \mH$ be an arbitrary hypothesis that satisfies $h(x_j) = z_{j}$ for all $j \in [m_1]$
\State $\C_{\mH} \leftarrow \C_{\mH} \cup \{h\}$
\EndFor
\State \Return $\C_\mH$
\end{algorithmic}
\end{algorithm}

\begin{algorithm}[!htbp]
\caption{$\textsc{Filter}(\mH, \D^{(t)}, \C_{\mH}, \OPT')$}
\label{algo:filter}
\begin{algorithmic}[1]
\State Sample $m_2 = O(\frac{(k+d)\log(kd/\eps\delta)}{\eps^2})$ data points $S^{(t)}_2 = \{(x_j, y_j)\}_{j \in [m_2]}$ from $\D^{(t)}$
\State $\mH' \leftarrow \mH$
\For{$h \in \C_{\mH}$}
\For{$\mI \subseteq [k]$}
\If{$\sum_{i \in \mI}p_t(i) \geq 1/2$ \textbf{and} $\frac{\sum_{j \in [m_2]} \mathsf{1}\{x_j \in \D_{\mI} \wedge h(x_j) \neq y_j\}}{\sum_{j \in [m_2]} \mathsf{1}\{x_j \in \D_{\mI}\}} \geq \OPT' + 8\eps $} \Comment{$\D_\mI:= \cup_{i\in \mI}\D_i$}
\State $\mH' \leftarrow \mH' \backslash \{h' \in \mH: h(x) = h'(x) \, \forall x \in S_{1}^{(t)}\}$  
\EndIf
\EndFor
\EndFor
\State \Return $\mH'$
\end{algorithmic}
\end{algorithm}

\begin{algorithm}[!htbp]
\caption{$\textsc{Estimate}(f_t)$}
\label{algo:estimate}
\begin{algorithmic}[1]
\For{$i = 1,2,\ldots, k$}
\State Sample $m_3 = O(\log(kd/\eps\delta)/\eps^2)$ data points $S_{3, i}^{(t)}$ from $\D_i$ 
\State $\hat{\ell}_{\D_i}(f_t) \leftarrow \Pr_{(x, y)\sim S_{3, i}^{(t)}}[f_t(x) \neq y]$
\State $\ell_t(i) \leftarrow \max\{\hat{\ell}_{\D_i}(f_t), \OPT' - \alpha\}$ \label{line:truncate} \Comment{$\alpha$ is the error of $\textsc{MultiLearnerOracle}$}
\EndFor
\State \Return $\ell_t$ 
\end{algorithmic}
\end{algorithm}

\subsection{Analysis}

Given an (infinite) hypothesis class $\mH$, let $\Delta(\mH)$ be all distributions over $\mH$ with finite support. Our goal is to prove
\begin{lemma}[Boosting framework]
\label{lem:boosting}
Suppose $\OPT' \in [\OPT -\eps, \OPT + \eps]$ and $\textsc{MultiLearnerOracle}$ is an $(\alpha, \delta/16T)$-multi-distribution learner whose output $f_t \in \Delta(\mH)$. 
Let $T = \log(k)(\alpha/\eps)^2$, then with probability at least $1-\delta$, $\textsc{BoostLearner}$ guarantees
\begin{align*}
\max_{i \in [k]}\ell_{\D_i}(f) \leq \OPT + 32 \eps.
\end{align*}
\end{lemma}

We devote to prove Lemma \ref{lem:boosting} in the rest of this section, and we always make the assumptions that $\OPT' \in [\OPT-\eps, \OPT+\eps]$ and $\textsc{MultiLearnerOracle}$ is an $(\alpha, \delta/32T)$-multi-distribution learner whose output $f_t \in \Delta(\mH)$. We further assume $\eps \leq \alpha/32$, otherwise we do not need $\textsc{BoostLearner}$.

We first state the guarantee of $\textsc{ConstructCover}$. 
\begin{lemma}[Guarantee of $\textsc{ConstructCover}$, adapted from Lemma 3.3 of \cite{alon2019limits}]
\label{lem:cover}
For any $t \in [T]$, with probability at least $1-\delta/32T$, $\C_{\mH}$ is an $\eps$-cover of $\mH$. 
Moreover, for any hypothesis $h \in \mH$, let $h' \in \C_{\mH}$ be the hypothesis with the same projection over $S_{1}^{(t)}$, we have
\[
\Pr_{x\sim \D^{(t)}} [h(x) \neq h'(x)] \leq \eps.
\]
\end{lemma}

We next provide the guarantee of $\textsc{Filter}$.
\begin{lemma}[Guarantee of $\textsc{Filter}$, Part 1]
\label{lem:filter1}
For each $t \in [T]$, with probability at least $1-\delta/16T$, we have $h^{*} \in \mH'$.
\end{lemma}
\begin{proof}
For each $t \in [T]$, we condition on the high probability event of Lemma \ref{lem:cover}.
Suppose $h^{*}_{\C} \in \C_{\mH}$ has the same projection as $h^{*}$ on $S_{1}^{(t)}$, it suffices to prove $h^{*}_{\C} \in \mH'$. By Lemma \ref{lem:cover}, we have
\begin{align}
\Pr_{x\sim \D^{(t)}} [h^{*}(x) \neq h^{*}_{\C}(x)] \leq \eps. \label{eq:filter4}
\end{align}
For any set $\mI \subseteq [k]$ with $\sum_{i\in \mI}p_t(i) \geq 1/2$, we have
\begin{align}
\frac{\sum_{i \in \mI}p_t(i)\ell_{\D_i}(h^{*}_{\C})}{\sum_{i\in \mI}p_t(i)} \leq \frac{\sum_{i \in \mI}p_t(i)\ell_{\D_i}(h^{*}) +\eps}{\sum_{i\in \mI}p_t(i)} \leq \OPT + 2\eps \label{eq:filter1}
\end{align}
Here the first step follows from Eq.~\eqref{eq:filter4}, the second step holds since $\ell_{\D_i}(h^{*}) \leq \OPT$ ($\forall i \in [k]$) and $\sum_{i\in \mI}p_t(i) \geq 1/2$.

Next, we have
\begin{align*}
&~ \Pr\left[\frac{\sum_{j \in [m_2]} \mathsf{1}\{x_j \in \D_{\mI} \wedge h^{*}_{\C}(x_j) \neq y_j\}}{\sum_{j \in [m_2]} \mathsf{1}\{x_j \in \D_{\mI}\}} \geq \OPT' + 8\eps\right]\\
\leq &~ \Pr\left[\frac{\sum_{j \in [m_2]} \mathsf{1}\{x_j \in \D_{\mI} \wedge h^{*}_{\C}(x_j) \neq y_j\}}{\sum_{j \in [m_2]} \mathsf{1}\{x_j \in \D_{\mI}\}} \geq \OPT + 7\eps\right]\\
\leq &~ \Pr\left[\sum_{j \in [m_2]} \mathsf{1}\{x_j \in \D_{\mI}\} \leq \frac{1}{4}m_2\right]\\
&~ + \Pr\left[\frac{\sum_{j \in [m_2]} \mathsf{1}\{x_j \in \D_{\mI} \wedge h^{*}_{\C}(x_j) \neq y_j\}}{\sum_{j \in [m_2]} \mathsf{1}\{x_j \in \D_{\mI}\}} \geq \OPT + 7\eps \mid  \sum_{j \in [m_2]} \mathsf{1}\{x_j \in \D_{\mI}\} \geq \frac{1}{4}m_2 \right] \\
\leq &~ \exp(-m_2/8) + \exp( - 2 \cdot (m_2/4) \cdot (5\eps)^2 ) \\
\leq &~ 2^{-k}\cdot \frac{\delta}{32T}. 
\end{align*}
The first step follows from $\OPT \leq \OPT' + \eps$, the third step follows from Chernoff bound, $\sum_{i\in \mI}p_t(i) \geq 1/2$ and Eq.~\eqref{eq:filter1}.
The last step follows from the choice of $m_2 \geq \Omega(k\log(kd/\eps\delta)/\eps^2)$. 

Taking a union bound over all subsets $\mI \subseteq [k]$, we have 
\[
\Pr[h^{*} \in \mH'] = \Pr[h^{*}_{\C} \in \mH'] \geq 1 - 2^{k}\cdot 2^{-k}\cdot \frac{\delta}{32T} \geq 1-\frac{\delta}{32T}.
\]
This finishes the proof.
\end{proof}

\begin{lemma}[Guarantee of $\textsc{Filter}$, Part 2]
\label{lem:filter2}
For each $t \in [T]$, with probability at least $1-\delta/16T$, it holds that for every hypothesis $h \in \mH'$ and every set $\mI \subseteq [k]$, if $\sum_{i\in \mI}p_t(i) \geq 1/2$, then
\[
\frac{\sum_{i \in \mI}p_t(i) \ell_{\D_i}(h)}{\sum_{i \in \mI}p_t(i)} \leq \OPT + 16\eps.
\]
\end{lemma}
\begin{proof}
For each $t \in [T]$, we condition on the high probability event of Lemma \ref{lem:cover}.
For each $h\in \C_{\mH}$, if there exists $h' \in \mH$ that has the same projection as $h$ on $S_{1}^{(t)}$, and there exists a subset $\mI \subseteq [k]$ with $\sum_{i\in \mI}p_t(i)\geq 1/2$, such that 
\begin{align}
\frac{\sum_{i \in \mI}p_t(i)\ell_{\D_i}(h')}{\sum_{i\in \mI}p_t(i)} \geq \OPT + 16\eps \label{eq:filter2}
\end{align}
then we prove $h$ would be removed from $\mH'$ with high probability.

On the same subset $\mI$, we have
\begin{align}
\label{eq:filter3}
\frac{\sum_{i \in \mI}p_t(i)\ell_{\D_i}(h)}{\sum_{i\in \mI}p_t(i)} \geq \frac{\sum_{i \in \mI}p_t(i)\ell_{\D_i}(h') - \eps}{\sum_{i\in \mI}p_t(i)} \geq \OPT + 14\eps \geq \OPT'+13\eps
\end{align}
where the first step holds from Lemma \ref{lem:cover}, the second step holds since $\sum_{i\in \mI}p_t(i) \geq 1/2$, the third step follows from Eq.~\eqref{eq:filter2}, and the last step follows from $\OPT'\leq \OPT +\eps$.

Now, we have
\begin{align*}
&~ \Pr\left[\frac{\sum_{j \in [m_2]} \mathsf{1}\{x_j \in \D_{\mI} \wedge h(x_j) \neq y_j\}}{\sum_{j \in [m_2]} \mathsf{1}\{x_j \in \D_{\mI}\}} < \OPT' + 8\eps\right]\\
\leq &~ \Pr\left[\sum_{j \in [m_2]} \mathsf{1}\{x_j \in \D_{\mI}\} < \frac{1}{4}m_2\right]\\
&~ + \Pr\left[\frac{\sum_{j \in [m_2]} \mathsf{1}\{x_j \in \D_{\mI} \wedge h(x_j) \neq y_j\}}{\sum_{j \in [m_2]} \mathsf{1}\{x_j \in \D_{\mI}\}} < \OPT' + 8\eps \mid  \sum_{j \in [n]} \mathsf{1}\{x_j \in \D_{\mI}\} \geq \frac{1}{4}m_2 \right] \\
\leq &~ \exp(-m_2/8) + \exp( - 2 \cdot (m_2/4) \cdot (5\eps)^2 ) \\
\leq &~ (kd/\eps^2\delta)^{-d}  \cdot \frac{\delta}{32T}.
\end{align*}
The second step follows from Chernoff bound, $\sum_{i\in \mI}p_t(i) \geq 1/2$ and Eq.~\eqref{eq:filter3}. The third step holds from the choice $m_2 \geq \Omega(d\log(kd/\eps\delta)/\eps^2)$.

Take a union bound over $\C_\mH$ and note that $|C_{\mH}| \leq (kd/\eps^2\delta)^d$ by Sauer–Shelah Lemma (see Lemma \ref{lem:ss}), we complete the proof. 
\end{proof}

Next we make some observations on the output $f_t$ of $\textsc{MultiLearnerOracle}$.
\begin{lemma}
\label{lem:f_t}
For each $t \in [T]$, with probability at least $1 -\delta/4T$, we have
\begin{itemize}
    \item $\max_{i\in [k]}\ell_{\D_i}(f_t) \leq \OPT + \alpha$
    \item For any set $\mI \subseteq [k]$ with $\sum_{i\in \mI}p_t(i) \geq 1/2$, $\frac{\sum_{i \in \mI}p_t(i) \ell_{\D_i}(f_t)}{\sum_{i \in \mI}p_t(i)} \leq \OPT +16\eps$.
\end{itemize}
\end{lemma}
\begin{proof}
We condition on the high probability event of Lemma \ref{lem:filter1} and Lemma \ref{lem:filter2}.
The first claim follows from
\begin{align*}
\max_{i\in [k]}\ell_{\D_i}(f_t) \leq \argmin_{h \in \mH'}\max_{i\in [k]}\ell_{\D_i}(h) + \alpha = \OPT + \alpha
\end{align*}
The first step follows from the guarantee of $\textsc{MultiLearnerOracle}$, the second step holds since $h^{*}\in \mH'$.

For the second claim, since $f_t \in \Delta(\mH')$, we can write $f_t =\sum_{j}q_j h_j$ for some $h_j \in \mH'$ and $\sum_{j}q_j = 1$. Then for any set $\mI \subseteq [k]$ with $\sum_{i\in \mI}p_t(i) \geq 1/2$, we have
\begin{align*}
\frac{\sum_{i \in \mI}p_t(i) \ell_{\D_i}(f_t)}{\sum_{i \in \mI}p_t(i)} = \frac{\sum_{j}q_j\sum_{i \in \mI}p_t(i) \ell_{\D_i}(h_j)}{\sum_{i \in \mI}p_t(i)} \leq \sum_{j}q_j (\OPT+16\eps) = \OPT + 16\eps.
\end{align*}
Here the first step holds since 
\[
\ell_{\D_i}(f_t) = \Pr_{(x, y)\sim \D_i}[f_t(x) \neq y] = \sum_{j}q_j\Pr_{(x, y)\sim \D_i}[h_j(x) \neq y] = \sum_{j}q_j \ell_{\D_i}(h_j).
\]
and the second step holds due to Lemma \ref{lem:filter2}.
\end{proof}

Finally, we make some observations on the loss vector $\ell_t$ constructed by $\textsc{Estimate}$.
\begin{lemma}[Guarantee of $\textsc{Estimate}$]
\label{lem:loss-vector}
For any $t \in [T]$, with probability at least $1 -\frac{\delta}{2T}$, we have 
\begin{itemize}
\item $\ell_t(i) \geq \ell_{\D_i}(f_t) - \eps$
\item $\ell_t(i) \in [\OPT - 2\alpha, \OPT + 2\alpha] $
\item $\sum_{i\in [k]}p_t(i)\ell_{t}(i) \leq \OPT + 20 \eps$
\end{itemize}
\end{lemma}
\begin{proof}
For each $t \in [T]$, we condition on the high probability event of Lemma \ref{lem:f_t}.
For each $i \in [k]$, since $m_3\geq \Omega(\log(kd/\eps\delta)/\eps^2)$, by Chernoff bound, with probability at least $1 - \frac{\delta}{32kT}$, the empirical loss $\hat{\ell}_{\D_i}(f_t)$ is $\eps$-close to the population loss $\ell_{\D_i}(f_t)$. Taking a union bound, we have
\begin{align}
\hat{\ell}_{\D_i}(f_t) \in [\ell_{\D_i}(f_t) - \eps, \ell_{\D_i}(f_t) + \eps] \quad \forall i\in [k]
\label{eq:loss1}
\end{align}
holds with probability at least $1 - \frac{\delta}{32T}$.

For the first claim, we have
\begin{align*}
\ell_t(i) = \max\{\hat{\ell}_{\D_i}(f_t), \OPT' - \alpha\} \geq \hat{\ell}_{\D_i}(f) \geq \ell_{\D_i}(f_t) - \eps.
\end{align*}

For the second claim, we have
\[
\ell_t(i) = \max\{\hat{\ell}_{\D_i}(f_t), \OPT' - \alpha\} \geq \OPT' - \alpha \geq \OPT-2\alpha
\]
and 
\[
\ell_t(i) = \max\{\hat{\ell}_{\D_i}(f_t), \OPT' - \alpha\} \leq \max\{\ell_{\D_i}(f_t),   \OPT - \alpha\} +\eps \leq \OPT + 2\alpha,
\]
where the second step follows from Eq.~\eqref{eq:loss1} and $\OPT' \leq \OPT+\eps$, the third step holds since $\ell_{\D_i}(f_t) \leq \OPT + \alpha$ (Lemma \ref{lem:f_t}).

For the last claim, w.l.o.g., we can assume $\ell_{\D_1}(f_t) \geq \cdots \geq \ell_{\D_k}(f_t)$. Let $k' \in [k]$ be the smallest index such that $\sum_{i \leq k'}p_t(i)\geq 1/2$. We divide into two cases based on the value of $\ell_{D_{k'}}(f_t)$.

If $\ell_{D_{k'}}(f_t) \geq \OPT -\alpha$, then we have
\begin{align}
\sum_{i\in [k]}p_t(i)\ell_{t}(i) = &~ \sum_{i\leq k'}p_t(i)\ell_{t}(i) + \sum_{i\geq k'+1}p_t(i)\ell_{t}(i) \notag \\
= &~  \sum_{i\leq k'}p_t(i)\cdot \max\{\hat{\ell}_{\D_i}(f_t), \OPT' - \alpha\} + \sum_{i\geq k'+1}p_t(i)\cdot \max\{\hat{\ell}_{\D_i}(f_t), \OPT' - \alpha\}\label{eq:loss2}
\end{align}
For the first term, we have
\begin{align}
\sum_{i\leq k'}p_t(i)\cdot \max\{\hat{\ell}_{\D_i}(f_t), \OPT' - \alpha\}\leq &~ \sum_{i\leq k'}p_t(i)\cdot (\max\{\ell_{\D_i}(f_t), \OPT - \alpha\} + \eps) \notag \\
= &~\sum_{i\leq k'}p_t(i)\cdot (\ell_{\D_i}(f_t) + \eps)\notag \\
\leq &~\sum_{i\leq k'}p_t(i)\cdot (\OPT + 17\eps).\label{eq:loss3}
\end{align}
The first step follows from $\OPT' \leq \OPT+\eps$ and Eq.~\eqref{eq:loss1}, the second step follows from the assumption that $\ell_{D_{k'}}(f_t) \geq \OPT -\alpha$, the third step holds due to Lemma \ref{lem:f_t}.

For the second term, we have
\begin{align}
\sum_{i\geq k'+1}p_t(i)\cdot \max\{\hat{\ell}_{\D_i}(f_t), \OPT' - \alpha\} \leq &~ \sum_{i\geq k'+1}p_t(i) (\max\{\ell_{\D_i}(f_t), \OPT-\alpha\} + \eps)\notag \\
\leq &~ \sum_{i\geq k'+1}p_t(i) (\max\{\ell_{\D_{k'}}(f_t), \OPT-\alpha\} + \eps)\notag \\
= &~ \sum_{i\geq k'+1}p_t(i) (\ell_{\D_{k'}}(f_t) + \eps )\notag \\
\leq &~ \sum_{i\geq k'+1}p_t(i) (\OPT + 17 \eps) \label{eq:loss4}
\end{align}
The first step follows from $\OPT' \leq \OPT+\eps$ and Eq.~\eqref{eq:loss1}, the third step follows from the assumption that $\ell_{D_{k'}}(f_t) \geq \OPT -\alpha$, the last step follows from $\ell_{\D_{k'}}(f_t) \leq \ell_{\D_{i}}(f_t)$ ($i \leq k'$) and Lemma \ref{lem:f_t}.

Combining Eq.~\eqref{eq:loss2}\eqref{eq:loss3}\eqref{eq:loss4}, we have proved
$
\sum_{i\in [k]}p_t(i)\ell_{t}(i)\leq \OPT+17\eps.
$

If $\ell_{D_{k'}}(f_t) < \OPT -\alpha$, then we have
\begin{align*}
\sum_{i\in [k]}p_t(i)\ell_{t}(i) = &~ \sum_{i\leq k'-1}p_t(i)\ell_{t}(i) + \sum_{i\geq k'}p_t(i)\ell_{t}(i) \notag \\
= &~  \sum_{i\leq k'-1}p_t(i)\cdot \max\{\hat{\ell}_{\D_i}(f_t), \OPT' - \alpha\} + \sum_{i\geq k'}p_t(i)\cdot \max\{\hat{\ell}_{\D_i}(f_t), \OPT' - \alpha\}\\
\leq &~ \sum_{i\leq k'-1}p_t(i) \cdot (\OPT+\alpha + \eps) + \sum_{i\geq k'}p_t(i)\cdot  (\OPT-\alpha+\eps) \\
\leq &~ \OPT +\eps.
\end{align*}
Here the third step holds since (1) $\hat{\ell}_{\D_i}(f_t) \leq \ell_{\D_i}(f_t) +\eps \leq \OPT +\alpha +\eps$ for any $i \in [k'-1]$ (Lemma \ref{lem:f_t}), and (2) $\hat{\ell}_{\D_i}(f_t) \leq \ell_{\D_i}(f_t) +\eps \leq \ell_{D_{k'}}(f_t) +\eps \leq \OPT-\alpha +\eps$ for any $i \geq k'$ due to the assumption $\ell_{D_{k'}}(f_t) < \OPT -\alpha$. The last step holds since $\sum_{i\leq k'-1}p_t(i)< 1/2$. This completes the proof for all three claims.
\end{proof}

Finally, we can prove our main Lemma \ref{lem:boosting}.
\begin{proof}[Proof of Lemma \ref{lem:boosting}]
We condition on the high probability events of Lemma \ref{lem:cover} -- \ref{lem:loss-vector}.
For any $i \in [k]$, due to the regret guarantee of MWU, we have
\begin{align*}
(\OPT + 20\eps) T \geq &~ \sum_{t\in [T]} \langle p_t, \ell_t\rangle \geq \sum_{t \in [T]}\ell_t(i) - 2\sqrt{\log(k)T} \cdot 4\alpha\\
\geq &~ \sum_{t \in [T]}\ell_{\D_{i}}(f_t) - \eps T - 8\sqrt{\log(k)T}\alpha.
\end{align*}
The first step follows from the third claim of Lemma \ref{lem:loss-vector}, the second step follows from the regret guarantee of MWU (Lemma \ref{lem:mwu}) and the width is at most $4\alpha$ (the second claim of Lemma \ref{lem:loss-vector}). The third step follows from the first claim of Lemma \ref{lem:loss-vector}.

Hence, we have
\begin{align*}
\ell_{\D_{i}}(f) = \frac{1}{T}\sum_{t \in [T]}\ell_{\D_{i}}(f_t) \leq \OPT + 21\eps + 8\sqrt{\log(k)/T}\alpha \leq \OPT + 32\eps.
\end{align*}
Here the first step holds since
\[
\ell_{\D_{i}}(f) = \Pr_{(x, y)\sim \D_i}[f(x)\neq y] = \frac{1}{T}\sum_{t\in [T]}\Pr_{(x, y)\sim \D_i}[f_t(x)\neq y] =  \frac{1}{T}\sum_{t \in [T]}\ell_{\D_{i}}(f_t).
\]
and the last step holds due to the choice of $T = \log(k) (\alpha/\eps)^2$. We complete the proof here.
\end{proof}

\section{Final algorithm}
\label{sec:final}

$\textsc{BoostLearner}$ gives a way of converting a weak multi-distribution learner into a strong one. Recursively evoking itself, we have 
\begin{lemma}[Recursive application of $\textsc{BoostLearner}$]
\label{lem:recursive}
Let $\mH$ be a hypothesis class of VC dimension at most $d$ and $\D_1, \ldots, \D_k$ be $k$ distributions. Given $\OPT'\in [\OPT-\eps, \OPT+\eps]$, for any integer $r \geq 1$, there is an algorithm with sample complexity
\[
O\left(\frac{(k+d)(\log(k))^{2r}\log(kd/\eps\delta)}{\eps^{2 (1+1/r)}}\right)
\]
and with probability at least $1-\delta$, returns a hypothesis $f \in \Delta(\mH)$ such that
\begin{align*}
\max_{i \in [k]}\ell_{\D_i}(f) \leq \OPT + 32\eps.
\end{align*}
\end{lemma}
\begin{proof}
We prove by induction. For $r = 1$, we run $\textsc{BoostLearner}$ with $\textsc{MultiLearnerOracle}$ selecting an arbitrary hypothesis in $\mH'$.
In this way, $\textsc{MultiLearnerOracle}$ takes $0$ additional samples and $\alpha = 1$. By Lemma \ref{lem:boosting}, the output $f \in \Delta(\mH)$ satisfies
\[
\max_{i\in [k]}\ell_{\D_i}(f) \leq \OPT + 32\eps.
\]
The total number of sample it takes equals
\[
T \cdot (m_1+m_2+k m_3) = O(\log(k)\eps^{-2} \cdot (k+d)\eps^{-2}\log(kd/\eps\delta)) = O((k+d)\eps^{-4}\log(k)\log(kd/\eps\delta)).
\]

Suppose the claim continues to hold up to $r$, then for $r+1$, we run $\textsc{BoostLearner}$ and set $\textsc{MultiLearnerOracle}$ to be the level $r$ algorithm, with error parameter $\eps' = \eps^{\frac{r}{r+1}}$ and confidence parameter $\delta' = \delta/16T$.
At each round $t\in [T]$, the VC dimension of $\mH'$ is at most $d$, and with high probability, $h^{*}\in \mH'$ (Lemma \ref{lem:filter1}). Therefore, $\textsc{MultiLearnerOracle}$ draws
\[
m = O\left(\frac{(k+d)(\log(k))^{2r}\log(kd/\eps'\delta')}{(\eps')^{2 (1+1/r)}}\right) = O\left(\frac{(k+d)(\log(k))^{2r}\log(kd/\eps\delta)}{\eps^{2}}\right)
\]
samples, and with probability at least $1-\delta/16T$, the hypothesis $f_t \in \Delta(\mH)$ it returns has error at most 
\[
\alpha  = 32 \eps' = 32 \eps^{\frac{r}{r+1}}.
\]
Therefore, by Lemma \ref{lem:boosting}, we obtain an $(32\eps, \delta)$-multi-distribution learner and its sample complexity equals
\begin{align*}
T (m_1+m_2+km_3+ m) = &~ O\left(\log(k)\alpha^{2}\eps^{-2} \cdot \frac{(k+d)(\log(k))^{2r}\log(kd/\eps\delta)}{\eps^{2}}\right)\\
= &~ O\left(\frac{(k+d)(\log(k))^{2r+2}\log(kd/\eps\delta)}{\eps^{2(1+\frac{1}{r+1})}}\right).
\end{align*}
This completes the proof.
\end{proof}


The algorithm described in lemma \ref{lem:recursive} still requires the prior knowledge of $\OPT$. Next, we give a way of removing this prior knowledge. 
\begin{lemma}[Remove prior knowledge of $\OPT$]
\label{lem:opt}
For any $\kappa \geq 2$, suppose there exists an algorithm that receives $\OPT' \in [\OPT-\eps, \OPT+\eps]$, returns a hypothesis of error at most $32\eps$ and has sample complexity $g(k, d, \delta)\eps^{-\kappa}$. 
Then there is an algorithm of sample complexity $g(k, d, \eps^2\delta/80)\cdot \eps^{-\kappa}\log(1/\eps)$ and returns a hypothesis of error at most $33\eps$. Here $g(k, d, \delta)$ is a function of $k, d, \delta$.
\end{lemma}
\begin{proof}
We prove the following claim by induction: For any $r \geq 1$, let $\delta_{r} = \eps\delta/ 2^r 80$, there is an algorithm that draws
$
O\left(g(k, d, \delta_r) \cdot 40r \cdot \eps^{-\kappa - \frac{1}{\kappa^{r-1}}}\right)
$
samples and obtains a hypothesis $f$ such that $\max_{i \in [k]}\ell_{\D_i}(f) \leq \OPT + 33\eps$, without knowing $\OPT$.

Let $\ALG$ be the input algorithm that requires prior knowledge of $\OPT$. For the base case $r = 1$, we instantiate $B = 1/\eps$ threads of $\ALG$, with $\OPT' = b \cdot \eps$ ($b \in [B]$), and obtain $\{f_{b}\}_{b \in [B]}$. 
We select the best hypothesis among $\{f_b\}_{b \in [B]}$, by drawing $O(\log(k/\eps\delta)/\eps^2)$ samples from each distribution and estimating the empirical loss of $\{f_b\}_{b\in [B]}$. The output hypothesis $f$ satisfies
\[
\max_{i\in [k]}\ell_{\D_i}(f) \leq \min_{b \in [B]}\max_{i\in [k]}\ell_{\D_i}(f_b) + \eps \leq \OPT + 33\eps.
\]
since one of the guess $\OPT'$ has error at most $\eps$. The sample complexity equals 
$g(k,d, \eps\delta/2)\eps^{-\kappa-1} + O(k\log(k/\eps\delta)/\eps^2) \leq  g(k,d, \delta_1)\eps^{-\kappa-1}$.

Suppose the claim continues to hold for $r$, then for $r+1$, the algorithm first runs the level $r$ algorithm with error parameter $\eps' = \eps^{(\kappa +\frac{1}{\kappa^{r}})/(\kappa +\frac{1}{\kappa^{r-1}})}$. In particular, it draws
\[
n_1 = g(k, d, \delta_r/2) \cdot 40r \cdot (\eps')^{-\kappa - \frac{1}{\kappa^{r-1}}} = g(k, d, \delta_{r+1}) \cdot 40r \cdot \eps^{-\kappa - \frac{1}{\kappa^{r}}}
\]
samples and obtains a hypothesis $f'$ of error at most $33\eps'$. It then draws $n_2 = O(\log(k/\eps\delta)/\eps^2)$ samples from each distribution and estimates the empirical loss $\hat{\ell}_{\D_i}(f')$ of $f'$ on each distribution $\D_i$ ($i\in [n]$).
Next, it instantiates $B = 33\eps' / \eps$ threads of $\ALG$, with $\OPT' = \max_{i\in [k]}\hat{\ell}_{\D_i}(f') - b \eps$ ($b\in [B]$), and obtains $\{f_{b}\}_{b \in [B]}$.  The number of samples taken in this step equals
\begin{align*}
n_3 =&~ 33 (\eps'/\eps) \cdot g(k, d, \eps\delta/80) \cdot \eps^{-\kappa} = 33 g(k,d,\delta_1) \cdot \eps^{-\kappa - 1 +  (\kappa +\frac{1}{\kappa^{r}})/(\kappa +\frac{1}{\kappa^{r-1}})} \\
\leq &~ 33 g(k,d,\delta_{r+1})\cdot \eps^{-\kappa -\frac{1}{\kappa^r}}.
\end{align*}
The final output $f$ is the best hypothesis among $f'$ and $\{f_b\}_{b\in [B]}$, measured with their empirical loss. 
The sample complexity of the algorithm equals 
\begin{align}
n_1 + k n_2 + n_3 \leq g(k, d, \delta_{r+1}) \cdot 40(r+1)\cdot \eps^{-\kappa - \frac{1}{\kappa^{r}}}.\label{eq:opt1}
\end{align}
For the output hypothesis $f$, if $\max_{i\in [k]}\ell_{\D_i}(f') \leq \OPT + 30\eps$, then we have
\begin{align}
\max_{i \in [k]}\ell_{\D_i}(f) \leq \max_{i\in [k]}\ell_{\D_i}(f')+ \eps \leq \OPT+31\eps\label{eq:opt2}.
\end{align}
Otherwise, if $\max_{i\in [k]}\ell_{\D_i}(f') \geq \OPT + 30\eps$, the one of the guess $\{\max_{i\in [k]}\hat{\ell}_{\D_i}(f') - b \eps\}_{b \in [B]}$ of $\OPT'$ is $\eps$-close to $\OPT$, and therefore, we have
\begin{align}
\max_{i \in [k]}\ell_{\D_i}(f) \leq \min_{b\in [B]}\max_{i\in [k]}\ell_{\D_i}(f_b) + \eps \leq \OPT + 33\eps  \label{eq:opt3}
\end{align}
where the last step follows from the guarantee of $\ALG$. Combining Eq.~\eqref{eq:opt1}\eqref{eq:opt2}\eqref{eq:opt3}, we complete the induction. Taking $r = \log(1/\eps)$, we finish the proof.
\end{proof}

Combining Lemma \ref{lem:opt} and Lemma \ref{lem:recursive} (taking $r = \omega(1)$) we complete the proof of Theorem \ref{thm:multi}.

\newpage
\bibliographystyle{alpha}
\bibliography{ref}

\end{document}